\documentclass{ecai} 

\usepackage{latexsym}
\usepackage{amssymb}
\usepackage{amsmath}
\usepackage{amsthm}
\usepackage{booktabs}
\usepackage{enumitem}
\usepackage{graphicx}
\usepackage{color}
\usepackage{algorithm}
\usepackage{algorithmic}

\newtheorem{theorem}{Theorem}

\newtheorem{definition}{Definition}

\newcommand{\BibTeX}{B\kern-.05em{\sc i\kern-.025em b}\kern-.08em\TeX}

\begin{document}


\begin{frontmatter}


\paperid{2377}


\title{Detecting Hidden Triggers: Mapping \\ Non-Markov Reward Functions to Markov}


\author[A]{\fnms{Gregory}~\snm{Hyde}\orcid{0000-0001-5608-5224}\thanks{Corresponding Author. Email: gregory.m.hyde.th@dartmouth.edu.}}
\author[A]{\fnms{Eugene}~\snm{Santos, Jr.}\orcid{0000-0003-2923-5113}}

\address[A]{Thayer School of Engineering at Dartmouth College, Hanover, NH, USA}


\begin{abstract}
Many Reinforcement Learning algorithms assume a Markov reward function to guarantee optimality. However, not all reward functions are Markov. This paper proposes a framework for mapping non-Markov reward functions into equivalent Markov ones by learning specialized reward automata, Reward Machines. Unlike the general practice of learning Reward Machines, we do not require a set of high-level propositional symbols from which to learn. Rather, we learn \emph{hidden triggers}, directly from data, that construct them. We demonstrate the importance of learning Reward Machines over their Deterministic Finite-State Automata counterparts given their ability to model reward dependencies. We formalize this distinction in our learning objective. Our mapping process is constructed as an Integer Linear Programming problem. We prove that our mappings form a suitable proxy for maximizing reward expectations. We empirically validate our approach by learning black-box, non-Markov reward functions in the Officeworld domain. Additionally, we demonstrate the effectiveness of learning reward dependencies in a new domain, Breakfastworld.
\end{abstract}

\end{frontmatter}

\section{Introduction}
Reinforcement Learning (RL) is traditionally formulated as a sequential decision-making task modeled under a Markov Decision Process (MDP) \cite{bellman_markovian_1957}. This implies that both the reward and transition functions are Markov, depending solely on the current state and action. However, in complex environments, these functions are not always known and cannot always be assumed to be Markov. In fact, it has been shown that there are behaviors that cannot be encoded by any Markov reward function under certain state representations \cite{abel_expressivity_2022}. To address non-Markov dynamics, researchers often extend the state representation to enforce the Markov property, though this is usually at the cost of overcomplicating the state representation or making it excessively sparse. For example, Mnih et al. \cite{mnih_playing_2013} stacked subsequent frames from the Atari games \cite{bellemare_arcade_2013} to encode a suitable representation for training Deep Q-Networks (DQNs). Similarly, Recurrent DQNs (RDQNs) have been used to encode state histories for similar purposes \cite{hausknecht_deep_2015}. However, these sparse, deep models obscure the interpretation of rewards. To better understand them, it is crucial to examine their dynamics under less sparse and more interpretable forms.


To this end, automata have been used to compactly represent the hidden or historical features of non-Markov rewards. While Deterministic Finite-State Automata (DFAs) have traditionally served this role \cite{brafman_ltlfldlf_2018, gaon_reinforcement_2019, furelos-blanco_induction_2021}, more recently, specialized reward automata, Reward Machines (RMs), have been used \cite{xu_joint_2020, rens_learning_2020, rens_online_2021, dohmen_inferring_2022}. We highlight two distinctions between DFAs and RMs: 1.) DFAs are inherently less expressive than RMs as they only \emph{accept} or \emph{reject} traces. Consequently, when rewards are complex, multiple DFAs must be learned \cite{gaon_reinforcement_2019}. In contrast, RMs can encode complex reward behavior in a single automaton. 2.) DFAs are only capable of representing patterns over $H = (S \times A)^{*}$, the history of states and actions. Conversely, RMs can represent patterns over $H = (S \times A \times \mathbb{R})^{*}$, which include rewards. This affords us a unique learning objective where rewards serve as memory allowing us to model their dependencies.

However, prior work on learning RMs assumes access to $P$, a set of high-level propositions, and $L: S \times A \times S \rightarrow 2^P$, a labeling function mapping transitions to instantiations in $P$. This makes much of this prior work ungeneralizable. We, instead, infer $L$ by identifying minimal patterns in $H=(S \times A \times \mathbb{R})^{*}$, called \emph{hidden triggers}, that differentiate reward outcomes. Rather than learn RMs directly, we map observed non-Markov rewards onto the cross-product of $U$, the state space of a candidate RM, and $S$, the observed state space, to form a Markov representation called the Abstract Reward Markov Decision Process (ARMDP). We infer an ARMDP, and thus its decomposable RM, by solving an Integer Linear Program (ILP), lending itself to powerful off-the-shelf discrete optimization solvers. 
\vspace{2pt} \\
\textbf{Contributions:} We introduce a novel algorithm for mapping non-Markov reward functions to equivalent Markov ones. We show how by leveraging $H=(S \times A \times \mathbb{R})^{*}$ in our learning objective we can dramatically expedite learning in cases with interdependent reward signals. We validate our approach by demonstrating our ability to learn black-box RMs from the Officeworld domain, originally introduced in the first RM paper \cite{icarte_using_2018}. We then evaluate the representative power of the ARMDP by extending DQN agents with ARMDP features to compare their learning profiles against various sparse RDQN variants. Finally, we validate our approach to represent complex reward function behavior in a new domain, Breakfastworld, featuring black-box RMs with interdependent reward signals. We show how by leveraging $H=(S \times A \times \mathbb{R})^{*}$ in these experiments we can expedite learning. Our main contributions are three-fold:

\begin{enumerate}
    \item We introduce a novel algorithm for learning RMs without assumed access to $P$ and $L$,
    \item We prove that the ARMDP is a suitable proxy for maximizing reward expectations under non-Markov rewards,
    \item We demonstrate the effectiveness of learning RMs under multiple interdependent reward signals.
\end{enumerate}

\section{Related Work}
Decision-making with non-Markov reward functions is often represented as a Non-Markov Reward Decision Process (NMRDP) \cite{thiebaux_decision-theoretic_2006}. An NMRDP is similar to an MDP except that the reward function, $R: H \rightarrow \mathbb{R}$, maps state and action histories, $H=(S \times A)^*$, to reward values. However, as $H$ can become unwieldy due to the curse of history, more succinct representations have been encoded using automata. The intuition of this approach lies in the fact that automata are finitely represented and offer a level of interpretation due to their structured form and symbolic representation.

Traditionally, the automata of choice has been the DFA, yielding promising results in identifying Markov representations for NMRDPs \cite{brafman_ltlfldlf_2018,gaon_reinforcement_2019,furelos-blanco_induction_2021}. More recently, specialized reward automata, RMs, have also been used \cite {icarte_using_2018}. While DFAs operate over an input alphabet and emit \emph{true} or \emph{false} over traces, RMs extend this expressiveness by emitting reward values along each transition. RMs are particularly appealing as they offer succinct encoding of non-Markov reward behavior and task decomposition. They have also been extended by a broad range of RL algorithms that exploit RM structure to learn effective or optimal (in the discrete case) policies \cite{camacho_ltl_2019}. These include but are not limited to, specialized Q-Learning (QRM), Counterfactual experiences (CRM), Hierarchical Reinforcement Learning (HRM) and Reward Shaping (RSRM) for RMs \cite{camacho_ltl_2019,icarte_reward_2022}.

Whether using a DFA or an RM, various strategies for learning them exist. One prominent strategy for learning automata is to use Angluin's L* algorithm \cite{angluin_learning_1987}. L* has been applied for both DFA \cite{gaon_reinforcement_2019} and RM \cite{rens_learning_2020,rens_online_2021,dohmen_inferring_2022} learning. Other strategies include Inductive Logic Answer Set Programming (ILASP) \cite{furelos-blanco_induction_2021} for DFAs and SAT solvers \cite{xu_joint_2020} or Tabu search \cite{toro_icarte_learning_2019} for RMs. Non-automata-based learning of non-Markov rewards has been accomplished via Linear Temporal Logic ($LTL$), \cite{li_reinforcement_2017,brafman_ltlfldlf_2018,camacho_learning_2021}. LTL can also be later translated into automata. Camacho et al. provide a helpful pipeline for transforming LTL formulae into RMs \cite{camacho_ltl_2019}.

A notable limitation of these prior works is that they assume access to a set of high-level symbols (the input language for the automata) and the relationship of the underlying state space to those symbols. In RM literature, these are denoted $P$ and $L: S \times A \times S \rightarrow 2^P$ where $P$ is a set of symbols, separate from the NMRDP, and $L$ is a mapping of transitions from the NMRDP to instantiations in $P$, respectively. A similar assumption is often made for DFA learning, with exceptions being the work of Gaon and Brafman \cite{gaon_reinforcement_2019} as well as Christoffersen et al. \cite{christoffersen_learning_2020} who learn symbolic representations from $H=(S \times A)^{*}$ to build their automata. Gaon and Brafman learn their DFA by applying L*, however, their approach was shown to be sample inefficient. Christoffersen et al. improved upon this by solving a discrete optimization problem for DFA learning, as outlined by Shvo et al. \cite{shvo_interpretable_2020}, which regularizes automata based on their size to better handle noise and enhance generalization. We improve upon these approaches by, instead, learning an RM without assumed access to $P$ and $L$. We show how, by learning RMs, we can model reward dependencies in a single automata by representing patterns in $H=(S \times A \times \mathbb{R})^*$.

\section{Preliminaries}
 \subsection{Reinforcement Learning and Non-Markov Reward}
RL is a type of self-supervised learning in which an agent learns to interact with its environment to maximize its long-term reward expectation \cite{sutton98}. It is an iterative process in which the agent receives feedback in the form of rewards or punishments to adjust its behavior. RL is formalized using the MDP. An MDP is a 6-tuple $(S, A, T, R, \gamma, \rho)$, where $S$ is the state space, $A$ is the action space, $T:S \times A \times S \rightarrow [0,1]$ is a Markov probabilistic transition function, $R:S \times A \times S \rightarrow \mathbb{R}$ is a Markov reward function mapping states, actions and states to rewards, $\gamma \in [0,1)$ is the discount factor and $\rho:S \rightarrow [0,1]$ is the initial state distribution. Briefly, NMRDPs are similar to MDPs in that they share $(S, A, T, R, \gamma, \rho)$; however, $R: H \rightarrow \mathbb{R}$ maps histories of states and actions to rewards, instead.

The solution for an MDP is a policy, $\pi: S \rightarrow A$, that maps states to actions. The value of a policy for any state, $s \in S$, at time point $t$ is the expected discounted return, defined as follows:
\begin{equation} \label{eq:val}
    V_{\pi}(s) = \mathbb{E}_\pi \left[ \sum_{k=0}^{\infty}\gamma^{k}r_{t+k+1} \mid s_t = s \right]
\end{equation}
where $\mathbb{E}_{\pi}$ is the expected reward from starting in $s$ and following $\pi$. The optimal policy, denoted $\pi_{*}$, is determined by solving for the optimal state value function:
\begin{equation} \label{eq:valop}
    V_{*}(s) = \max_{\pi} V_{\pi}(s)
\end{equation}
for all $s \in S$. We can write the optimal state-action value function in terms of the state value function as follows:
\begin{equation} \label{eq:qop}
    Q_{*}(s,a) = \mathbb{E}_{\pi} \left[r_{t+1} + \gamma V_{*}(s_{t+1}) \mid s_t = s, a_t = a \right]
\end{equation}
for all $s\in S$ and $a\in A$. Hence:
\begin{equation} \label{eq:optimize}
    V_{*}(s) = \max_{a}Q_{*}(s,a)
\end{equation}
Equation \eqref{eq:optimize} is often solved through a process called Value Iteration \cite{sutton98}. The optimal policy is realized by selecting actions according to $\max_{a}Q_{*}(s,a)$, otherwise known as a greedy policy. A popular model-free approach for learning Equation (\ref{eq:optimize}) is Q-Learning \cite{watkins_q-learning_1992}:
\begin{equation} \label{eq:ql}
    q(s_t,a_t) \xleftarrow \alpha \left[ r_{t+1} + \gamma V(s_{t+1}) \right]
\end{equation}
where $q(s_t,a_t)$ is updated slowly according to $\alpha \in [0,1)$. When updating Equation (\ref{eq:ql}) in an online fashion, typically an $\epsilon$-greedy strategy is employed where the agent follows a greedy policy but chooses a random action $\epsilon-$percent of the time. Both strategies are guaranteed to converge to the optimal policy in the limit as each state-action pair is visited infinitely often, but only under the assumption that $T$ and $R$ are Markov. Otherwise, expectations have no stable solution.

\subsection{Reward Machines}
RMs are a specialized type of automata imposed over an underlying state and action space, $S$ and $A$ \cite{icarte_using_2018}. RMs operate over a set of propositional symbols, $P$, representing high-level events, separate from $S$. RMs are defined as a 5-tuple, $(U, u_1, F, \delta_{u}, \delta_r)$, where $U$ is a finite state space, $u_1 \in U$ is the initial state, $F$ is a set of terminal states where $F \cap U = \emptyset$, $\delta_{u}: U \times 2^{P} \rightarrow U \cup F$ is a transition function and $\delta_{r}:U \times 2^{P} \times (U \cup F) \rightarrow \mathbb{R}$ is a reward function. A labeling function, $L:S \times A \times S \rightarrow 2^{P}$, connects $S$ and $A$ to the RM by mapping transitions to instantiations in $P$.
\begin{figure}
  \centering
  \includegraphics[width=0.35\textwidth]{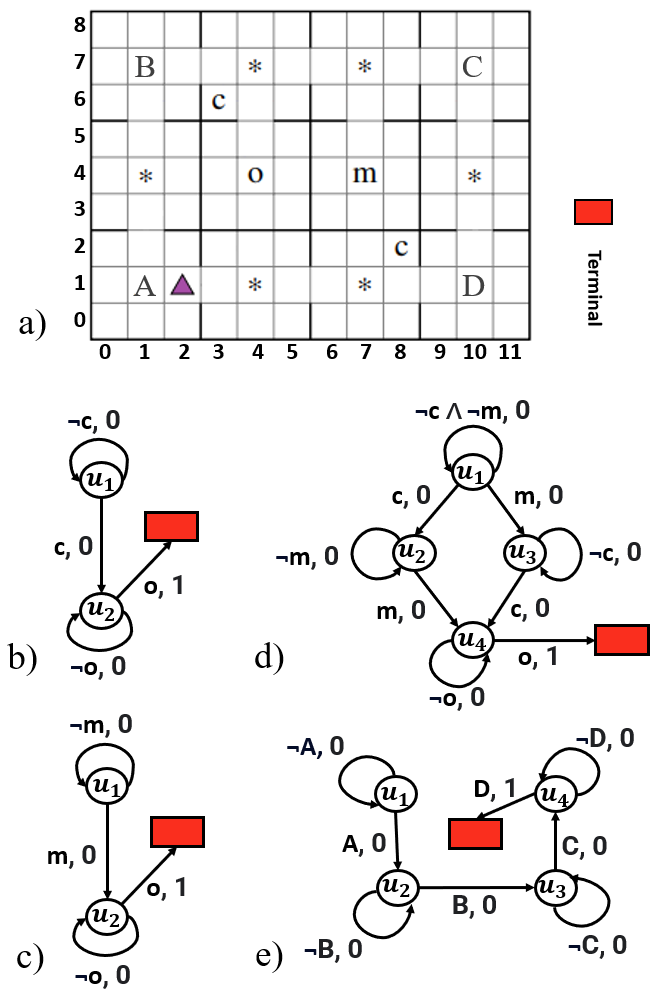}
  \caption{The Officeworld domain (a) with four Reward Machines, (b) deliver coffee to office, (c) deliver mail to office, (d) deliver coffee and mail to office and (e) patrol task sequencing A, B, C and D.}
  \label{fig:ow}
\end{figure}
RMs are capable of modeling reward dependencies based on histories that can be represented as a regular language (e.g., loops, conditions, and interleaving). RMs were initially demonstrated on the  Officeworld domain \cite{icarte_using_2018} shown in Figure \ref{fig:ow} (a). Tasks (b-e) represent RMs sensitive to $P=\{c, m, o, *, A, B, C, D\}$ where $c$ is a coffee tile, $m$, the mail tile, $o$, the office tile, $*$, a plant tile, and $A, B, C, D$ are room IDs. $\neg$, $\lor$ and $\land$ are logical operators for \emph{negation}, \emph{or}, and \emph{and}, respectively. The cross-product of an NMRDP and an RM can be used to form a Markov representation for predicting both $R$ and $T$.

\section{Approach}
We wish to map non-Markov reward signals into succinct Markov representations. We do so by inferring \emph{minimal} RMs that represent hidden or historical features that render the representation Markov. To make our approach generalizable to practical decision-making scenarios, we drop the assumed access to $P$ and $L$. For simplicity of notation, we drop notation on terminal states, $F$. 

\paragraph{Problem Statement} Given a finite set of finite-length trajectories, $T_o$, sampled from an NMRDP, $(S, A, T, R, \gamma, \rho)$, we learn a \emph{minimal} RM $(U, u_1, \delta_u, \delta_r)$ that renders rewards Markov to $S$, $A$ and $U$. \vspace{6pt}

While $R$ in the NMRDP is unknown, we assume it can be sampled from and is representable as an RM. Hereafter, we will refer to $R$ as a black-box RM. Our approach for solving this problem is to learn an ARMDP, which is simply the MDP cross-product between an RM and an NMRDP that enables the extraction of the associated RM. Learning of the ARMDP is formulated as an ILP. A \textit{minimal} RM is an RM with the smallest hidden state space $|U|$, and the least number of transition edges between different RM states.

The following sections are organized as follows: We begin by defining the ARMDP as the cross-product space between an RM and an NMRDP in Section \ref{sec:amdp}. Section \ref{sec:form} describes a passive learning approach for learning ARMDPs from $T_o$, by solving an ILP. We later make this an active learning approach in Section \ref{sec:active} by incorporating a Q-Learning agent. Section \ref{sec:extract} briefly outlines how to extract an RM from an ARMDP. Finally, Section \ref{sec:theo} describes the theoretical insights of ARMDPs as a proxy for the NMRDP.

\subsection{Abstract Reward Markov Decision Process} \label{sec:amdp}

An ARMDP is an MDP with additional properties sufficient for extracting/learning a valid RM. 

\begin{definition}\label{def:ARMDP}
    Given an NMRDP $(S, A, T, R, \gamma, \rho)$, RM $(U, u_1, \delta_u, \delta_r)$, and labeler, $L: S \times A \times S \rightarrow 2^P$, an ARMDP is an MDP, $(\tilde{S}, A, \tilde{T}, \tilde{R}, \gamma, \tilde{\rho})$, satisfying:
    \begin{enumerate}
        \item $\tilde{S} = S \times U$
        \item $\{u \mid (s, u) \in \tilde{S}, \tilde{\rho}((s, u)) > 0 \} = \{u_1\}$
        \item
            \begin{equation*}\label{eq:rewards}
                \resizebox{.99\linewidth}{!}{$
                    \displaystyle
                    \tilde{T}((s,u), a, (s', u')) = 
                    \begin{cases}
                        T(s, a, s') & \text{if\ } \delta_u(u, L(s, a, s')) = u' \\
                        0 & \text{otherwise}
                    \end{cases}$}
            \end{equation*}
        \item $\tilde{R}((s,u), a, (s', u')) = \delta_r(u, L(s, a, s'), u')$.
    \end{enumerate}
\end{definition}
\emph{Property 1} defines the state space of the ARMDP as the cross-product of $S$ (what we observe) and $U$ (what we do not observe). \emph{Property 2} imposes an initial state distribution compliant with $u_1 \in U$ being the initial RM state. \emph{Property 3} defines the transition function of the ARMDP as a composition function of $T$ and $\delta_u$. Finally, \emph{Property 4} defines the reward function of the ARMDP as $\delta_r$. 

The intuition behind the ARMDP is that if we assume rewards in $T_o$ are Markov to $S$ and $A$, conflicts are bound to arise. Consider $\delta_r: U \times 2^P \times U \rightarrow \mathbb{R}$, for some black-box RM, being rewritten as $\delta_r: U \times S \times A \times S \times U \rightarrow \mathbb{R}$ by invoking $L$. Because we do not observe $U$, we lack the specificity to predict $\delta_r$. Our goal is to find some candidate ARMDP, that renders rewards Markov to $S$, $A$, and $U$, by identifying \emph{hidden triggers} in $H=(S \times A \times \mathbb{R})^*$ that resolve all reward conflicts in $T_o$. 

\subsection{Formulation} \label{sec:form}
We infer candidate ARMDPs by solving an ILP that resolves reward conflicts over $T_{o}$. Let $\tau_m \in T_{o}$ be the m-th trace in $T_{o}$ of length $l_{m}$. Further, let $\tau_{m,n}=(s_n^m, a_n^m, r_{n+1}^m, s_{n+1}^m)$ be the n-th tuple of $\tau_m$. We define a reward conflict to be any two $\tau_{m,n}$ and $\tau_{\tilde{m}, \tilde{n}}$ where $(s_n^m, a_n^m, s_{n+1}^m) = (s_{\tilde{n}}^{\tilde{m}}, a_{\tilde{n}}^{\tilde{m}}, s_{\tilde{n}+1}^{\tilde{m}})$, but $r_{n+1}^m \neq r_{\tilde{n}+1}^{\tilde{m}}$. The motivation of our ILP is to map $T_{o}$ onto a valid ARMDP such that we resolve all reward conflicts observed in $T_o$. Intuitively, $\tilde{S}$ provides a degree of freedom over $U$ to do just that.


We define the matrix, $O = M \times N \times K \times K$, where $M = |T_{o}|$, $N$ reflects the largest $l_m$, $\forall \tau_m \in T_o$, and $K = |U|$ reflects the size of the finite state space for some candidate RM. Each $O_{m,n}$, then, is a $K \times K$ matrix that lays the foundation for \emph{Property 1} in the ARMDP:
\begin{equation*}
    O_{m,n} = \begin{pmatrix}
    O_{m,n,1,1} & O_{m,n,1,2} & \cdots & O_{m,n,1,K} \\
    O_{m,n,2,1} & O_{m,n,2,2} & \cdots & O_{m,n,2,K} \\
    \vdots & \vdots & \ddots & \vdots \\
    O_{m,n,K,1} & O_{m,n,K,2} & \cdots & O_{m,n,K,K}
\end{pmatrix}
\end{equation*}
Each $O_{m,n,i,j} \in O_{m,n}$ is a binary variable reflecting the decision to map $s_n^m$ to RM state $u_i \in U$ and $s_{n+1}^m$ to RM state $u_j \in U$. That is, if $O_{m,n,i,j} = 1$, then $\tau_{m,n} = (s_n^m, a_n^m, r_{n+1}^m, s_{n+1}^m)$ is mapped onto the corresponding ARMDP tuple of the form, $((s_n^m, u_i), a_n^m, r_{n+1}^m (s_{n+1}^m, u_j))$. As such, each $\tau_{m,n}$ introduces $K^2$ variables through its respective $O_{m,n}$ matrix. We constrain each $O_{m,n}$ to choose a single mapping by enforcing that $\forall m=1, 2, ..., M, \forall n = 1, 2, ..., l_m$:

\begin{equation} \label{eq:equal1}
    \sum_{i=1}^{K} \sum_{j=1}^{K} O_{m,n,i,j} = 1
\end{equation}
We further constrain $O_{m,n}$ where $n=1$ to satisfy the initial state requirement of \emph{Property 2} in the ARMDP. That is, $\forall m=1, 2, ..., M$:

\begin{equation} \label{eq:init}
    \sum_{j=1}^{K} O_{m,1,1,j} = 1
\end{equation}
Next, we ensure the continuity of each trajectory in the mapping process. That is, we encode agreement between subsequent mappings of all $O_{m,n}$ and $O_{m,n+1}$ pairs by enforcing the following constraint $\forall m=1, 2, ..., M, \forall n = 1, 2, ..., l_m - 1, \forall j = 1, 2, ..., K$:
\begin{equation} \label{eq:sub}
    \sum_{i=1}^{K} O_{m,n,i,j} = \sum_{j'=1}^{K} O_{m, n+1, j, j'}
\end{equation}
Equation (\ref{eq:sub}) simply forces $s_{n+1}^m$ to be mapped to the same RM state in both $\tau_{m,n}$ and $\tau_{m,n+1}$. Equations (\ref{eq:equal1}-\ref{eq:sub}) form the scaffolding of our mapping.


Thus far, all constraints have been applied over individual or subsequent tuples of $T_{o}$. However, it is necessary to form a consensus on the mapping for all observed traces onto the ARMDP to ensure \emph{Property 3} and \emph{Property 4}. We begin with \emph{Property 3}. 

\emph{Property 3} encodes a transition function, $\tilde{T}$, constructed over components $T$, from the NMRDP, and $\delta_u$, from the RM. Even though $\delta_u$ is unknown, we know that it must be deterministic. That is, for any transition, $(s, a, s') \in S \times A \times S$, and for any $u_i \in U$, $\delta_u(u_i, L(s, a, s'))$ must map to a single $u_j \in U$. Therefore, we must constrain $O$ to agree on the assignments of $u_i$ and $u_j$ over shared transition tuples. We construct the following indicator variable:
\begin{equation*}
    I_{s, a, s', i, j} =    \left[ \sum_{m=1}^M \sum_{\substack{n=1, \\
    (s, a, s') = (s_n^m, a_n^m, s_{n+1}^m)}}^{l_m} O_{m,n,i,j} \geq 1 \right]
\end{equation*}
$I_{s, a, s', i, j}$ is an indicator variable representing the sum of all $O_{m, n, i, j}$ in $O$ sharing the transition tuple $(s, a, s')$ with assignments $i, j \in {1, 2, ..., K}$. When $I_{s, a, s', i, j} = 1$, it acts as a \emph{hidden trigger} that highlights the sensitivity of $L$ to $(s, a, s')$ and enforces the transition $\delta_u(u_j, L(s, a, s')) = u_j$, in the RM, thus governing $\tilde{T}$. We later consider the sensitivity of \emph{hidden triggers} to rewards in \emph{Property 4}. 

Let $(S,A,S)_{o} = \{ (s_n^m, a_n^m, s_{n+1}^m) \mid m \in \{1, 2, ..., M\}, n \in \{ 1, 2, ..., l_m\}\}$ be the set of all unique transitions in $T_{o}$. Then to enforce the determinism of $\delta_u$, we impose that $\forall (s, a, s') \in (S, A, S)_{o}, \forall i = 1, 2, ..., K$:
\begin{equation} \label{eq:ambig}
    \sum_{j = 1}^{K} I_{s, a, s', i, j} \leq 1
\end{equation}
Consider that for any $(s, a, s') \in (S,A,S)_{o}$ and for any $i \in \{1, 2, ..., K\}$ that when we constrain the sum $\forall j=1, 2, ..., K$, in Equation (\ref{eq:ambig}), to be less than or equal to 1, we limit $\delta_u(u_i, L(s, a, s'))$ to have at most one outcome. With this, Equation (\ref{eq:ambig}) satisfies the determinism of $\delta_u$ and \emph{Property 3} of the ARMDP.

The encoding of \emph{Property 4} enforces the resolution of reward conflicts in $T_{o}$. While conflicts exist under the perspective of $S \times A \times S$, by mapping conflicting reward tuples onto $S \times U \times A \times S \times U$ we can separate them. Similar to our encoding of \emph{Property 3}, we do this by constraining $O$ to agree on the assignments of $u_i$ and $u_j$ over shared reward tuples by way of indicator variables:
\begin{equation*}
    I_{s, a, r, s', i, j} =    \left[ \sum_{m=1}^M \sum_{\substack{n=1, \\
    (s, a, r, s') = \tau_{m,n}}}^{l_m} O_{m,n,i,j} \geq 1 \right]
\end{equation*}
$I_{s, a, r, s', i, j}$ is an indicator variable representing the sum over all $O_{m, n, i, j}$ in $O$ sharing the reward tuple $(s, a, r, s')$ with assignments $i, j \in \{1, 2, ..., K\}$. When $I_{s, a, r, s', i, j} = 1$, it imposes the following reward emission, $\delta_r(u_i, L(s, a, s'), u_j) = r$, in the RM. 

Let $r(s, a, s') = \{ r_{n+1}^m \mid m \in \{1, 2, ..., M\}, \hspace{2pt} n \in \{ 1, 2, ..., l_m\}, \\ (s_n^m, a_n^m, s_{n+1}^m) = (s, a, s')\}$ be the set of all observed rewards for some $(s, a, s')$. Then to enforce the determinism of $\delta_r$, we impose that $\forall (s, a, s') \in (S, A, S)_{o}, \forall i = 1, 2, ..., K, \forall j = 1, 2, ..., K$:
\begin{equation} \label{eq:rambig}
    \sum_{r \in r(s, a, s')} I_{s, a, r, s', i, j} \leq 1
\end{equation}
Consider that for any $(s, a, s') \in (S,A,S)_{o}$ and for any $i, j \in \{1, 2, ..., K\}$ that when we constrain the sum $\forall r \in r(s, a, s')$, in Equation (\ref{eq:rambig}), to be less than or equal to 1, we limit $\delta_r(u_i, L(s, a, s'), u_j)$ to have at most one reward outcome, thus resolving our reward conflicts in $T_o$. Importantly, $I_{s, a, r, s', i, j}=1$ contextualizes its corresponding \emph{hidden trigger}, $I_{s, a, s', i, j}$, with reward. When rewards are interdependent (e.g., some upstream reward conflict determines a downstream reward conflict), rewards serve as memory, hence, we capture relationships in $H=(S \times A \times \mathbb{R})^*$. By Equation (\ref{eq:rambig}), we satisfy the determinism of $\delta_r$ and \emph{Property 4} of the ARMDP. 


Equations (\ref{eq:equal1}-\ref{eq:rambig}) form the entirety of the constraints for ARMDP learning. However, there are potentially many ARMDP solutions for any given $T_{o}$. As such, we are motivated to derive simple and interpretable models for representing reward dynamics. Therefore, we orient our objective and learning procedure toward finding the \emph{minimal} ARMDP, and thus the \emph{minimal} RM, that satisfies $T_{o}$. This is done in two ways. First, for any inference of an ARMDP, we start by assuming $K = |U|= 2$, thus dramatically limiting our search space. If $K = 2$ is determined to be insufficient to resolve reward conflicts (ILP infeasibility) we increment $K$ by 1 and try again. This is known as a \emph{deepening} strategy employed in various related works \cite{christoffersen_learning_2020, camacho_learning_2021}. Second, we limit the number of \emph{hidden triggers} that dictate transitions between different RM states (i.e., $\delta_u(u_i, L(s, a, s')) = u_j$ where $u_i \neq u_j$). We do so by minimizing the following objective:

\begin{equation} \label{eq:obj}
    z = \sum_{(s, a, s) \in (S, A, S)_{o}}\sum_{i=1}^{K}\sum_{\substack{j=1, \\
    i \neq j}}^{K} I_{s, a, s', i, j}
\end{equation}

We solve Equations (\ref{eq:equal1}-\ref{eq:obj}) using out-of-the-box ILP solvers in Gurobi \cite{gurobi} to infer viable ARMDPs. We note that all notation in this section was constructed assuming $L:S \times A \times S \rightarrow 2^P$ though it is straightforward to rewrite it as being sensitive to $L:S \rightarrow 2^P$. We provide the full ILP in Appendix A.

We also note a structural assumption that we make in later experimentation. That is, we only use $O_{m, n, i, j}$ where $j \geq i$ (i.e., the upper triangle matrix of each $O_{m,n}$). The implication of this is that we assume the RM to never return to a previous state, once left. While this might seem to limit the expressiveness of any resulting ARMDP, we note that our ILP simply unrolls these cycles expanding $U$ as needed. Because $T_{o}$ represents a finite set of finite length trajectories, learned ARMDPs will be equivalently expressive, but larger in some cases. For all RMs considered in this work, this structural assumption holds resulting in the \emph{minimal} ARMDP being learned while dramatically reducing learning time. We provide more details on this assumption in Appendix B.

\subsection{Extracting RMs from ARMDPs} \label{sec:extract}
We provide a brief outline for extracting an RM from an ARMDP. Given that $K=|U|$ is determined by our \emph{deepening} strategy, we assume $U$ is given. Additionally, by Equation (\ref{eq:init}) we designate $u_1 \in U$ to be the initial state. To extract the rules of $\delta_u$ and $\delta_r$ we simply iterate $\forall m=1, 2, ... M, \forall n=1, 2, ..., l_m, \forall i = 1, 2, ..., K, \forall j = 1, 2, ..., K$. If $O_{m,n,i,j} = 1$ then $\delta_u(u_i, L(s_n^m, a_n^m, s_{n+1}^m)) = u_j$ and $\delta_r(u_i, L(s_n^m, a_n^m, s_{n+1}^m), u_j) = r_{n+1}^m$. 

\subsection{Theoretical Considerations} \label{sec:theo}

Our ILP is constructed such that it infers an ARMDP that resolves reward conflicts in $T_{o}$. Because the ARMDP is simply just an MDP, we can apply out-of-the-box RL algorithms to solve it. However, it remains to be shown how the ARMDP might be a suitable proxy for the NMRDP for the purpose of maximizing reward expectations. We start with the weighted Reward Sum (RS) for some $\tau_m \in T_{o}$:
\begin{equation}
    RS(\tau_m) = \sum_{n=1}^{l_m} r_{n+1}^m \cdot T(s_n^m, a_n^m, s_{n+1}^m)
\end{equation}
Given that we assume the form of $R$ to be an RM and that $\delta_u$ and $\delta_r$, in an RM, are deterministic functions, $RS(\tau_m)$ reliably captures the reward observations essential for constructing a reward expectation, assuming the trajectory sequence is preserved. Similarly, we construct a weighted Abstracted Reward Sum (ARS) for some  $\tau_m \in T_{o}$ under the representation of the ARMDP to be:
\begin{equation}
    ARS(\tau_m) = \sum_{n=1}^{l_m} \sum_{i=1}^{K} \sum_{j=1}^{K} r_{n+1}^m \cdot T(s_n^m, a_n^m, s_{n+1}^m) \cdot O_{m, n, i, j}
\end{equation}
We remind readers that every $O_{m, n, i, j}$ reflects a binary variable that, if 1, indicates $\delta_u(u_i, L(s_n^m, a_n^m, s_{n+1}^m)) = u_j$. By referencing \emph{Property 3} of the ARMDP, multiplying $O_{m, n, i, j}$ into $T(s_n^m, a_n^m, s_{n+1}^m)$ to determine $\tilde{T}$ should be evident. We introduce Theorem \ref{theo:equiv}:
\begin{theorem}\label{theo:equiv}
$RS(\tau_m) = ARS(\tau_m)$ $\forall \tau_m \in T_{o}$.
\end{theorem}
For space reasons, we provide the proof of Theorem \ref{theo:equiv} in Appendix C, but briefly state that by ensuring $RS(\tau_m) = ARS(\tau_m)$ $\forall \tau_m \in T_{o}$, the ARMDP is a suitable proxy for maximizing reward expectations for the NMRDP, assuming $T_o$ is representative. Furthermore, even if $T_o$ is not representative, we posit that any ARMDP derived from $T_o$ serves as a useful hypothesis for RL agents to interrogate in an active learning paradigm. We explore that intuition in the next section.



\subsection{RL for Non-Markov Rewards} \label{sec:active}
Our ILP formulation for learning ARMDPs is a form of passive learning. That is to say, while the ARMDP will sufficiently resolve reward conflicts found in $T_{o}$, $T_{o}$ may not be representative of the true reward dynamics. Therefore, we are interested in extending our approach into an active learning framework by applying Q-Learning over inferred ARMDPs. Our intuition is that while any instance of $T_o$, or any ARMDP inferred from it, may not be truly representative, Q-Learning agents might gradually expand $T_o$ in their pursuit to maximize reward expectations. In this way, each inferred ARMDP acts as a hypothesis that can be interrogated for consistency.


\begin{algorithm} [h]
    \caption{ARMDPQ-Learning}
    \textbf{Input}: NMRDP 
    \label{alg:algorithm}
    \begin{algorithmic}[1] 
        \STATE $\mathit{ARMDP} = \mathit{ARMDP}_0$
        \STATE $\pi_{\mathit{ARMDP}} = \pi_{0}$
        \STATE $T_{o} = []$
        \STATE $K = |U| = 2$
        \WHILE{$!done$}
            \STATE $\tau_m, \mathit{conflict} = \mathit{NMRDP}.\mathit{sim}(\mathit{ARMDP}, \pi_{\mathit{ARMDP}})$
            \IF{$\mathit{conflict}$}
                \STATE $T_{o}.append(\tau_m)$
                \STATE $\mathit{solved} = \mathit{False}$
                \WHILE{$!\mathit{solved}$}
                    \STATE $\mathit{ARMDP}, \mathit{solved} = \mathit{solve\_ILP}(T_{o}, K)$
                    \IF{$!solved$}
                        \STATE $K = K+1$
                    \ENDIF
                \ENDWHILE
                \STATE $\pi_{\mathit{ARMDP}} = \pi_{0}$
            \ELSE
                \STATE $\pi_{\mathit{ARMDP}} = \mathit{update\_qsa}(\tau_m)$
            \ENDIF
        \ENDWHILE
    \end{algorithmic}
\end{algorithm}

We present ARMDPQ-Learning, with pseudocode available in Algorithm \ref{alg:algorithm}. The agent applies a \emph{deepening} strategy by initially assuming $S$ and $A$ form an MDP (i.e., $|U|=1$). We denote this ARMDP as ARMDP$_0$ (line 1). $K=|U|=2$, is used for subsequent ARMDP inference. The agent policy, $\pi_{\mathit{ARMDP}}$, is initially set to $\pi_0$, a policy with all Q-values set to 0 (line 2). We use the NMRDP to simulate some $\tau_m$ under the perspective of the current ARMDP, selecting actions according to $\pi_{\mathit{ARMDP}}$ (line 6). If a conflict is found we add $\tau_m$ into $T_{o}$ (line 8) and solve for a new ARMDP (line 11). By adding only conflicting $\tau_m$, we maintain a smaller $T_{o}$ for more efficient ILP solving. If the problem is infeasible we increment $K$ by 1 (line 13) until there is a solution. As each ARMDP presents a new perspective, we reset $\pi_{\mathit{ARMDP}}$ back to $\pi_{0}$ (line 16). Assuming no conflict is found for $\tau_m$, we update $\pi_{\mathit{ARMDP}}$ with $\tau_m$ according to Q-Learning updating (line 18). Note that line 18 is shorthand. In reality, tuples are updated in an online fashion while simulating the trajectory.

\section{Experiments}
We briefly outline our experiments. Experiment 1 tests the efficacy of ARMDPQ-Learning in the Officeworld domain. Experiment 2 assesses the representative power of ARMDPs by comparing learning profiles of DQNs using $\tilde{S}$ versus various RDQN architectures only using $S$. Finally, Experiment 3 demonstrates the efficacy of learning RMs over DFAs given their ability to learn over $H=(S \times A \times \mathbb{R})^{*}$.

\subsection{Experiment 1 - Officeworld} \label{exp:1}
We validated our approach by applying ARMDPQ-Learning on the four black-box RMs (Tasks (b-e)) from the Officeworld domain visualized in Figure \ref{fig:ow}. We hid the RM as well as $P$ and $L$ from the agent. We refer readers to Appendix D for details on hyper-parameter selection, but note that we used an $\epsilon$-greedy policy and set $\epsilon$ to 0\% after 90,000 episodes to exploit learned behavior. Results for Experiment 1 can be seen in Figure \ref{fig:exp1_smol}.

\begin{figure}
    \centering
    \includegraphics[width=0.4\textwidth]{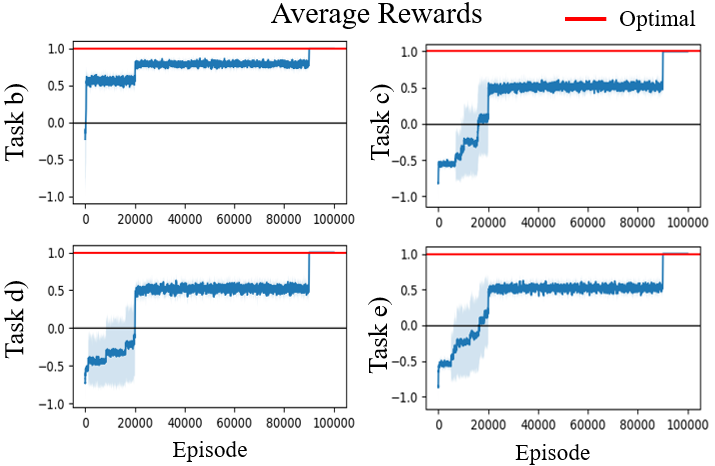}
    \caption{Average rewards with standard deviations, evaluated over 10 trials using ARMDPQ-Learning on Tasks (b-e) in the Officeworld domain.}
    \label{fig:exp1_smol}
\end{figure}

The agent learned to perform optimally in all tasks, suggesting that a meaningful representation was learned. Task (b), (c), and (e) all arrived at $|U|$ reflecting the ground truth model (2, 2, and 4, respectively). Conversely, Task (d) arrived at $|U| \approx 3.2$. Because the coffee and the mail tiles can be achieved in any order, occasionally the agent would lock on to one solution, and not the other, requiring less RM states. $T_{o}$ remained relatively small for all tasks, with $|T_{o}|\approx$ 4, 8, 39, and 44 for Tasks (b-e), respectively. We also tracked the cumulative solve times for all ILP solves under each task to be $\approx$ 0.1, 0.6, 4,200, and 12,000 seconds, respectively. We include an extension to Figure \ref{fig:exp1_smol} in Appendix D, containing figures for $|U|$, $|T_{o}|$ and cumulative solve times.

In all experiments, we noticed high variability in early performance reflected in high standard deviations between episodes 0-20,000. We attribute this to identifying the wrong pattern in $H$ during early conflict resolution. Interestingly, because these early solutions were so poor, they were easily testable by the Q-Learning agent. This lends some credence to our initial intuition of the Q-Learning agent serving as an important active learner capable of interrogating its own perspectives and updating in the face of new evidence.

\subsection{Experiment 2 - ARMDP Representations}
Previously, Christoffersen et al. compared the learning rates of Q-Learning agents over NMRDPs, using DFAs, to sparse, recurrent, deep RL models \cite{christoffersen_learning_2020}. However, in smaller discrete state spaces, Q-Learning has the advantage of updating state-action values directly given their simpler, tabular representation. In contrast, deep RL models must learn approximate solutions using complex computations over large batches of data. It remained to be shown how beneficial it could be to represent the state space of a DQN with $\tilde{S}$ instead of $S$. This would allow for a more even comparison of training profiles against RDQN architectures that have to learn a Markov representation over histories of $S$.

For each Task (b-e), in Experiment 1, we extracted $\tilde{S}$, from the finalized ARMDP. We used $\tilde{S}$ as the feature space for an Abstracted DQN (ADQN). As the ADQN explored each task we updated $\tilde{S}$ according to its respective ARMDP. In parallel, we trained recurrency-based models using $S$, namely, an RDQN, a Long-Short Term Memory DQN (LSTMDQN), and a Gated Recurrent DQN (GRUDQN). For a baseline, we also trained a DQN using $S$.

\begin{figure}
    \centering
    \includegraphics[width=0.4\textwidth, height=0.4 \textheight]{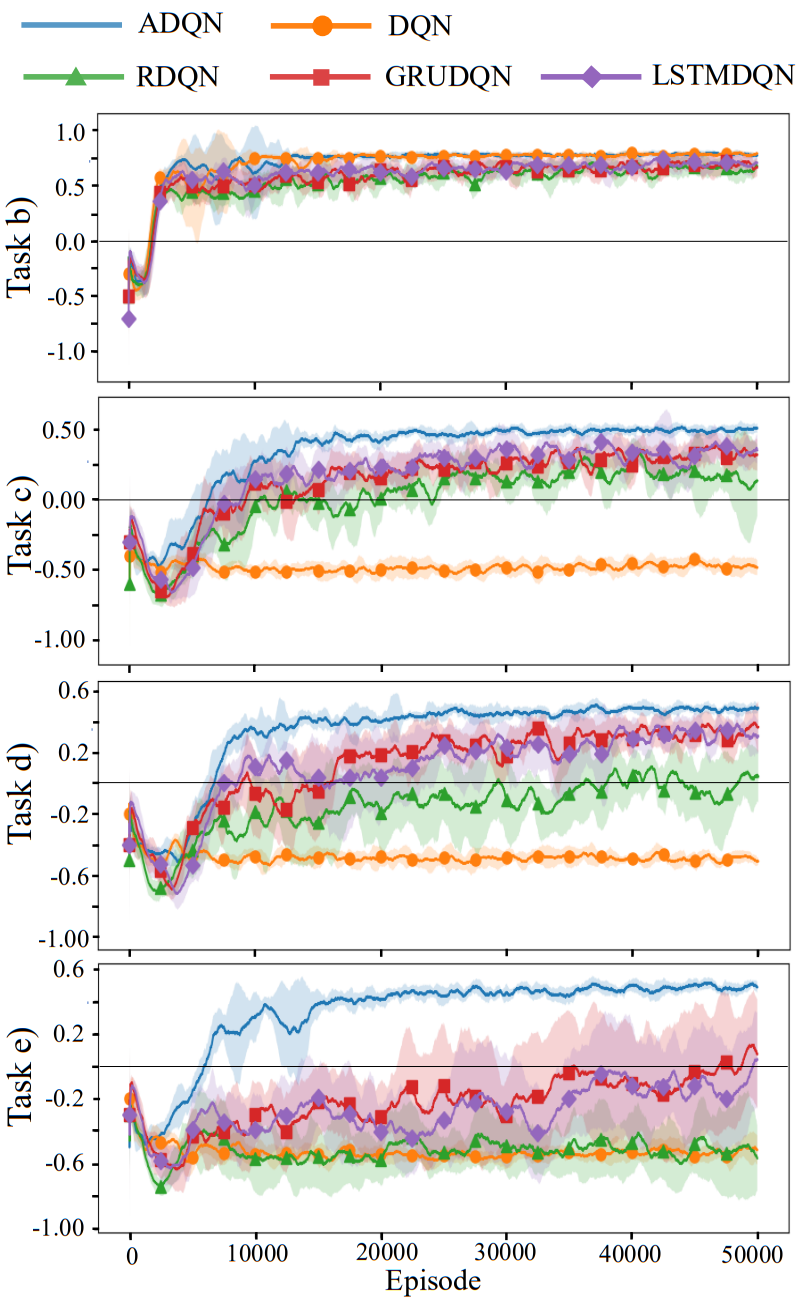}
    \caption{Average rewards with standard deviations, evaluated over 10 trials using DQN models on Tasks (b-e) in the Officeworld domain.}
    \label{fig:exp2}
\end{figure}
    
Rewards were very sparse for the Officeworld tasks and exploration via an $\epsilon$-greedy policy was often insufficient. As we didn't seek to demonstrate the efficacy of RL strategies under sparse rewards, we seeded each model with an optimal trajectory every 10th trajectory. Results for Experiment 2 can be seen in Figure \ref{fig:exp2}.

In all tasks, the ADQN model learned at the rate of, or out-learned, all other models. This became more apparent as the complexity of the task increased (Task (e) being the most complex). We note that for Task (b), the baseline DQN performed just as well as the other models only because the shortest path to the goal included the coffee tile. Interestingly, for Task (e), the RDQN performed just as poorly as the baseline DQN model while the LSTMDQN and GRUDQN models (while learning slowly) appeared to be uncovering the relevant hidden mechanisms of the underlying RM. We hypothesize that the \emph{forgetting} component of the GRUDQN and LSTDQN is a key differentiating factor in their learning ability. Still, we remind readers that uncovering $U$, as discussed in Experiment 1, involves a non-trivial cost before ADQN training can commence.

\subsection{Experiment 3 - Breakfastworld}
To evaluate the effectiveness of learning our RMs over $H=(S \times A \times \mathbb{R})^{*}$, rather than $H=(S \times A)^{*}$, we introduce a new domain, Breakfastworld, seen in Figure \ref{fig:bfw} (a). We include two tasks, Tasks (b) and (c), containing multiple interdependent reward signals. The domain uses $P=\{c, w, e, l\}$ representing propositions for \emph{cooking}, \emph{washing}, \emph{eating} and \emph{leaving}, respectively. Task (b) contains an RM wherein the agent is meant to cook, eat, and then leave. $(u_1, c, u_2)$ returns a reward of -0.1, reflecting a time cost. This time cost is experienced in a non-linear fashion as the agent grows bored and spending another turn cooking returns a reward of -0.3, $(u_2, c, u_3)$. However, this patience pays off in $(u_3, e, u_4)$, where eating a well cooked meal returns a reward of 2. In contrast, $(u_2, e, u_5)$ returns 1, reflecting an undercooked meal. Finally, $(u_4, l, term)$ and $(u_5, l, term)$ reflect the last reward dependency, returning the cumulative reward for each path in the RM, encouraging episode completion. Task (c) builds on Task (b) by adding a reward dependency for washing. The penalty for washing after cooking once, $(u_6, w, u_7)$, is less than after cooking twice, $(u_4, w, u_5)$.

\begin{figure}
  \centering
  \includegraphics[width=0.4\textwidth, height=0.4\textheight]{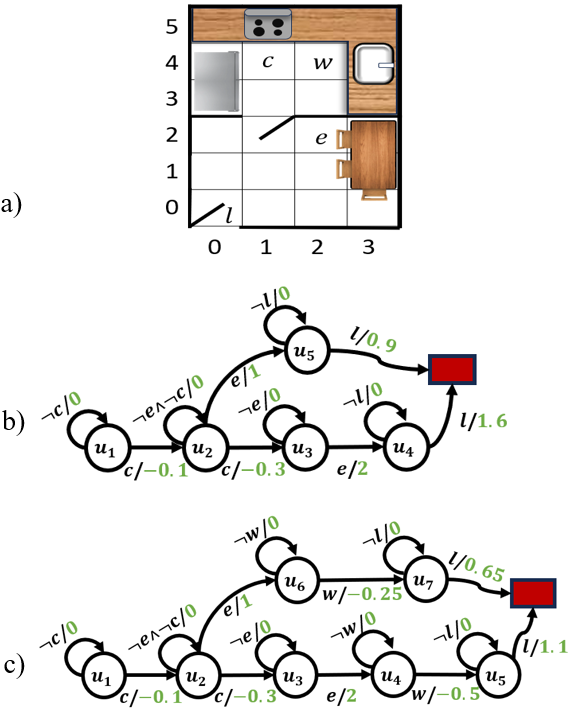}
  \caption{The Breakfastworld domain (a) with two RMs, (b) to cook, eat and then leave, (c) to cook, eat, wash, and then leave.}
  \label{fig:bfw}
\end{figure}

Both tasks reflect RMs more complex in terms of $|U|$ than Task (e) in Experiment 1. However, we demonstrate that RM learning benefits in cases where rewards are interdependent as upstream reward conflict resolution results in natural splits for downstream reward conflict resolution. Conversely, DFA learning suffers under such conditions as each reward conflict must be represented by its own DFA which is learned independently from any other DFA \cite{gaon_reinforcement_2019}, hence they only learn according to $H=(S \times A)^{*}$. 

We demonstrate how advantageous it can be that our ILP naturally learns according to $H=(S \times A \times \mathbb{R})^{*}$ by applying ARMDPQ-Learning to Tasks (b) and (c). We then compare this to applying ARMDPQ-Learning on RMs equivalent in structure to Task (b) and (c), but that only emit cumulative rewards on terminal transitions. All other rewards are made 0. We refer to Tasks (b) and (c), in Figure \ref{fig:bfw}, as Task (b) Full and Task (c) Full, respectively. Conversely, the RMs that only emit the cumulative rewards are referred to as Task (b) Cumulative and Task (c) Cumulative, respectively. The purpose of these cumulative tasks is to form a problem that produces reward conflicts only on the $l$ tile that can not be contextualized by upstream reward conflicts. This would be similar to learning DFAs to parse the different possible rewards on the $l$ tile only using $H=(S \times A)^{*}$. We show results for Experiment 3 in Table \ref{tab:exp3}.

In all solvable cases of ARMDPQ-Learning, the agent achieved optimal performance and identified the correct $|U|$, for each task. We include an extension figure to Table \ref{tab:exp3} in Appendix F containing such results. Using Task (b) Full our cumulative solve time for all ILPs was dramatically smaller ($\approx$ 4.4 seconds) than when using Task (b) Cumulative ($\approx$ 1,102.9 seconds). This can be attributed to Task (b) Full requiring only a fourth of the number of trajectories to solve than Task (b) Cumulative, likely due to the leveraged context of upstream rewards. More impressive were the results for Task (c) Full which resulted in the cumulative solve time of $\approx$ 153.8 seconds. In no instance of attempting to learn in Task (c) Cumulative did we arrive at a stable solution in 100,000 seconds ($>$27 hours).
\begin{table}
    \centering
    \begin{tabular}{lll}
        \toprule
        Task  & $|T_{o}|$ & Solve Time (s)\\
        \midrule
        Task (b) Full           & 6.2 $\pm$ 1.1     & 4.4 $\pm$ 3.3 \\
        Task (b) Cumulative     & 22.4 $\pm$ 8.8    & 1102.9 $\pm$ 1068.8  \\
        Task (c) Full           & 10.8 $\pm$ 3.0    & 153.8 $\pm$ 197.8 \\
        Task (c) Cumulative     & n/a               & n/a\\
        \bottomrule
    \end{tabular}
    \caption{Results of ARMDPQ-Learning evaluated on Tasks (b) and (c) from the Breakfastworld domain on 10 random trials.}
    \label{tab:exp3}
\end{table}

\section{Conclusions}
In this work, we introduced a formulation for learning RMs from non-Markov rewards by solving an ILP to learn an ARMDP. The ARMDP represents the cross-product between an RM and the underlying state space of an NMRDP. Importantly, we learn the ARMDP, and thus the RM, without knowledge of $P$ or $L$, by identifying \emph{hidden triggers} that govern the rules of $L:(S \times A \times S) \rightarrow 2^P$. We also demonstrated how effective it can be to learn our automata according to $H=(S \times A \times \mathbb{R})^{*}$, rather than $H=(S \times A)^{*}$, a limitation of DFAs. We proved that the ARMDP is a suitable proxy for maximizing the reward expectations of the NMRDP for $T_{o}$. We validated our approach by applying ARMDPQ-Learning to train Q-Learning agents over black-box RMs. We also evaluated the representative power of the ARMDP by representing the state space of a DQN as the state space of an ARMDP. We compared the training profiles of these DQNs to the training profiles of various RDQN architectures that had to learn Markov representations over the sparse histories they encode.

An important aspect of this work is that the RMs we learn are inherently more interpretable than the sparse representation of histories in $H=(S \times A \times \mathbb{R})^{*}$. This is in part because we learn \emph{minimal} RMs that encode $H$ under the finite, structured, and symbolic representation of an automaton, but also because interpreting an RM doesn't require a comprehensive understanding of $S$ and $A$ outside of what the labeler encodes. This is a consequence of the RM being entirely decomposable from the ARMDP. For example, the RM in Figure 1 Task (b) conveys the need to bring coffee to the office without ever needing to describe the sequence of states and actions to achieve that goal. Since RL is predicated on the assumption that the reward function is the most concise, robust, and transferable definition of a task \cite{ng_algorithms_2000}, we posit that RMs emerge as strong candidates for reward function representation and interpretation.

RM representations also carry important weight for the field of Inverse RL (IRL). In IRL, the goal is to learn a reward function from behavioral data that motivated that behavior \cite{ng_algorithms_2000,tan_inverse_2017,choi_inverse_2011}. While most approaches assume a Markov reward function, some of this work has been focused on learning rewards in Partially Observable MDPs (POMDPs) \cite{choi_inverse_2011} and in Semi MDPs (SMDPs) \cite{tan_inverse_2017}. We think that relaxing the Markov assumption in a principled way, analogous to this work, could yield insight into hidden abstractions made by a decision-maker. Moreover, given their interpretable form, it might shed light on what these hidden abstractions represent.

Lastly, we generally make note that this work is limited in that it assumes $S$ and $A$ are discrete. As such, we would like to consider automata extraction methods, in the future, using recurrent models \cite{tino_mealy_1995,omlin_dfas_1996,weiss_extracting_2020} as this would extend our approach into the continuous case. Weiss et al. have demonstrated promising results in this regard \cite{weiss_extracting_2020}. Similarly, this work assumes $R$ to be deterministic to some history. In scenarios where $R$ is stochastic or noisy, we posit that we might be able to handle such cases by sampling histories to assess if the reward distributions differ in some statistical sense.

\begin{ack}
This research was funded in part by the Air Force Office of Scientific Research Grant No. FA9550-20-1-0032 and the Office of Naval Research Grant No. N00014-19-1-2211.
\end{ack}



\bibliography{mybibfile}

\appendix
\onecolumn
\section*{\centering Appendix}
\section{Full ILP} \label{app:full}
\begin{align*}
     min \hspace{2pt} z = &  \sum_{(s, a, s) \in (S, A, S)_{o}}\sum_{i=1}^{K}\sum_{\substack{j=1, \\
    i \neq j}}^{K} I_{s, a, s', i, j} & & \hspace{2pt} (1)\\
    s.t. & & & \\
    & \sum_{i=1}^{K} \sum_{j=1}^{K} O_{m,n,i,j} = 1 & \forall m=1, 2, ..., M, \forall n = 1, 2, ... l_m & \hspace{2pt} (2) \\
    & \sum_{j=1}^{K} O_{m,1,1,j} = 1 & \forall m=1, 2, ..., M & \hspace{2pt} (3) \\
    & \sum_{i=1}^{K} O_{m,n,i,j} = \sum_{j'=1}^{K} O_{m, n+1, j, j'} & \forall m=1, 2, ..., M, \forall n = 1, 2, ..., l_m - 1, \forall j = 1, 2, ..., K & \hspace{2pt} (4) \\
    & \sum_{j = 1}^{K} I_{s, a, s', i, j} \leq 1 & \forall (s, a, s') \in (S, A, S)_{o}, \forall i = 1, 2, ..., K & \hspace{2pt} (5) \\
    & \sum_{r \in r(s, a, s')} I_{s, a, r, s', i, j} \leq 1 & \forall (s, a, s') \in (S, A, S)_{o}, \forall i = 1, 2, ..., K, \forall j = 1, 2, ..., K & \hspace{2pt} (6) \\
    & I_{s, a, s', i, j} =    \left[ \sum_{m=1}^M \sum_{\substack{n=1, \\
    (s, a, s') = (s_n^m, a_n^m, s_{n+1}^m)}}^{l_m} O_{m,n,i,j} \geq 1 \right] & \forall (s, a, s') \in (S, A, S)_{o}, \forall i = 1, 2, ..., K, \forall j = 1, 2, ..., K & \hspace{2pt} (7) \\
    & I_{s, a, r, s', i, j} = \left[ \sum_{m=1}^M \sum_{\substack{n=1, \\
    (s, a, r, s') = \tau_{m,n}}}^{l_m} O_{m,n,i,j} \geq 1 \right] & {\LARGE \substack{\forall (s, a, s') \in (S, A, S)_{o}, \forall r \in r(s, a, s'), \\ \hspace{30pt} \forall i = 1, 2, ..., K, \forall j = 1, 2, ..., K}} & \hspace{2pt} (8) \\
    & O_{m, n, i j} \in \{0, 1\} & {\LARGE \substack{\forall m=1, 2, ..., M, \forall n=1, 2, ..., l_m, \\ \hspace{13pt} \forall i = 1, 2, ..., K, \forall j = 1, 2, ..., K}} & \hspace{2pt} (9)
\end{align*}
The goal of our Integer Linear Program (ILP) formulation is to infer a valid Abstract Reward Markov Decision Process (ARMDP), as defined in Definition 1, in the main paper, from a set of observed trajectories, $T_o$, with observed reward conflicts. Reward conflicts are a result of an agent observing rewards under the lens of $S \times A \times S$, when in reality, rewards are sensitive to histories of the form $H=(S \times A \times \mathbb{R})^{*}$. We introduce some notation that our ILP leverages. Let, $\tau_m \in T_{o}$, be the m-th trace in $T_{o}$ of length $l_{m}$. Further, let $\tau_{m,n}=(s_n^m, a_n^m, r_{n+1}^m, s_{n+1}^m)$ be the n-th tuple of $\tau_m$. We define a reward conflict to be any two $\tau_{m,n}$ and $\tau_{\tilde{m}, \tilde{n}}$ where $(s_n^m, a_n^m, s_{n+1}^m) = (s_{\tilde{n}}^{\tilde{m}}, a_{\tilde{n}}^{\tilde{m}}, s_{\tilde{n}+1}^{\tilde{m}})$, but $r_{n+1}^m \neq r_{\tilde{n}+1}^{\tilde{m}}$.

We define the matrix, $O = M \times N \times K \times K$, where $M = |T_{o}|$, $N$ reflects the largest $l_m$, $\forall \tau_m \in T_{o}$, and $K = |U|$ reflects the size of the finite state space for some candidate Reward Machine (RM). $O_{m,n}$, then, is a $K \times K$ matrix that reflects the mapping of $\tau_{m,n} = (s_n^m, a_n^m, r_{n+1}^m s_{n+1}^M)$ onto the ARMDP, where both $s_n^m$ and $s_{n+1}^m$ are assigned an RM state. Additionally, we make use of two sets to gather necessary variables from $T_{o}$ for the construction of our ILP. These are $(S,A,S)_{o} = \{ (s_n^m, a_n^m, s_{n+1}^m) \mid m \in \{1, 2, ..., M\}, n \in \{ 1, 2, ..., l_m\}\}$ and $r(s, a, s') = \{ r_{n+1}^m \mid m \in \{1, 2, ..., M\}, n \in \{ 1, 2, ..., l_m\}, (s_n^m, a_n^m, s_{n+1}^m) = (s, a, s')\}$ where $(S,A,S)_{o}$ reflects the set of all observed transition tuples, $(s, a, s') \in S \times A \times S$, that exist in $T_{o}$, and $r(s, a, s')$ is the set of all observed rewards for some given $(s, a, s')$.

\section{j $\geq$ i Constraint} \label{app:ij}
This section describes the implications of imposing a $j \geq i$ constraint in the construction of $O$. Typically, $O_{m,n}$ is a $K \times K$ matrix that reflects the mapping of $\tau_{m,n} = (s_n^m, a_n^m, r_{n+1}^m s_{n+1}^M)$ onto the ARMDP, where both $s_n^m$ and $s_{n+1}^m$ are assigned an RM state:
\begin{equation*}
    O_{m,n} = \begin{pmatrix}
    O_{m,n,1,1} & O_{m,n,1,2} & \cdots & O_{m,n,1,K} \\
    O_{m,n,2,1} & O_{m,n,2,2} & \cdots & O_{m,n,2,K} \\
    \vdots & \vdots & \ddots & \vdots \\
    O_{m,n,K,1} & O_{m,n,K,2} & \cdots & O_{m,n,K,K}
\end{pmatrix}
\end{equation*}
When we impose a $j \geq i$, we only consider the upper right triangle matrix of $O_{m,n}$:
\begin{equation*}
    O_{m,n} = \begin{pmatrix}
    O_{m,n,1,1} & O_{m,n,1,2} & \cdots & O_{m,n,1,K} \\
    & O_{m,n,2,2} & \cdots & O_{m,n,2,K} \\
    & & \ddots & \vdots \\
    &  &  & O_{m,n,K,K}
\end{pmatrix}
\end{equation*}
The implication of this construction on the form of inferred ARMDP is that the underlying RM can never return to a prior state, once left. To reiterate our sentiments from the main paper, this does not limit the expressivity of the inferred ARMDP, it simply requires a larger $U$ in some instances. Consider the RMs in Figure \ref{app:figij}. Task (a) represents a patrol task identical to Task (e) in Figure 1, in the main paper. Even though Task (a) has atomic cycles (e.g., $u_1 \rightarrow u_1$), once a state is left, the RM never returns to it. For RMs like this, even with $j \geq i$ imposed, we still learn the \emph{minimal} RM. Conversely, Task (b) is a continuous patrol task, where on completion of a single patrol cycle, the RM begins again in $u_1$. Here, because the RM reuses previous states, the ARMDPs that we produce will continuously unroll these states and expand $U$ as needed. Instead of $u_1 \rightarrow u_2 \rightarrow u_3 \rightarrow u_4 \rightarrow u_1 ...$, the RM we infer will be represented as  $u_1 \rightarrow u_2 \rightarrow u_3 \rightarrow u_4 \rightarrow u_5 ... \hspace{2pt}$. However, because $T_{o}$ represents a finite set of finite trajectories, we will always be able to learn an ARMDP from it, though perhaps at the expense of excess computation in cycle decomposition. Regardless, all inferred ARMDPs will be equally expressive.

\begin{figure}
  \centering
  \includegraphics[width=0.9\textwidth]{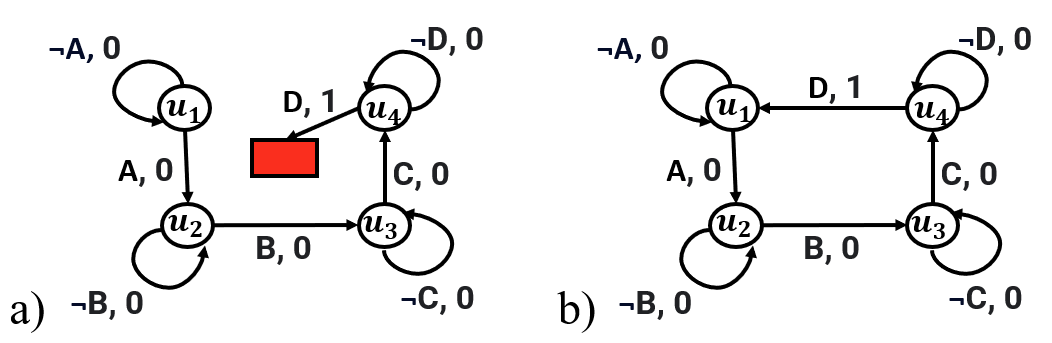}
  \caption{Patrol task where (a) patrols only once and (b) patrols continuously}
  \label{app:figij}
\end{figure}

\section{ARMDPs as NMRDP Proxies} \label{app:proof}
The goal of this section is to further elaborate on how the ARMDP is a sufficient proxy of the NMRDP for the purposes of maximizing reward expectations. We start by reintroducing the weighted Reward Sum (RS) for some $\tau_m \in T_{o}$:
\begin{equation}
    RS(\tau_m) = \sum_{n=1}^{l_m} r_{n+1}^m \cdot T(s_n^m, a_n^m, s_{n+1}^m)
\end{equation}
Given that we assume the form of $R$ to be an RM and that $\delta_u$ and $\delta_r$, in an RM, are deterministic functions, $RS(\tau_m)$ reliably captures the reward observations essential for constructing a reward expectation, assuming the trajectory sequence is preserved. Similarly, we construct a weighted Abstracted Reward Sum (ARS) for some  $\tau_m \in T_{o}$ under the representation of the ARMDP to be:
\begin{equation} 
    ARS(\tau_m) = \sum_{n=1}^{l_m} \sum_{i=1}^{K} \sum_{j=1}^{K} r_{n+1}^m \cdot T(s_n^m, a_n^m, s_{n+1}^m) \cdot O_{m, n, i, j}
\end{equation}
We remind readers that every $O_{m, n, i, j}$ reflects a binary variable that, if 1, allows us to conclude $\delta_u(u_i, L(s_n^m, a_n^m, s_{n+1}^m)) = u_j$. We reference \emph{Property 3} of the ARMDP:

\begin{equation*}
    \displaystyle
    \tilde{T}((s,u), a, (s', u')) = 
    \begin{cases}
        T(s, a, s') & \text{if\ } \delta_u(u, L(s, a, s')) = u' \\
        0 & \text{otherwise}
    \end{cases}
\end{equation*}
By multiplying $O_{m, n, i, j}$ into $T(s_n^m, a_n^m, s_{n+1}^m)$, $O_{m, n, i, j}$ serves as a logical gate to produce the dynamics of $\tilde{T}$. With this we reintroduce Theorem 1:
\newpage
\setcounter{theorem}{0}  
\begin{theorem}
$RS(\tau_m) = ARS(\tau_m)$ $\forall \tau_m \in T_{o}$.
\end{theorem}
\begin{proof}
\begin{align*}
    ARS(\tau_m) & = \sum_{n=1}^{l_m} \sum_{i=1}^{K} \sum_{j=1}^{K} r_{n+1}^m \cdot T(s_n^m, a_n^m, s_{n+1}^m) \cdot O_{m, n, i, j} \\
     & = \sum_{n=1}^{l_m} r_{n+1}^m \cdot T(s_n^m, a_n^m, s_{n+1}^m) \sum_{i=1}^{K} \sum_{j=1}^{K} O_{m, n, i, j} \\
     & = \sum_{n=1}^{l_m} r_{n+1}^m \cdot T(s_n^m, a_n^m, s_{n+1}^m) \cdot 1 \\
     & = \sum_{n=1}^{l_m} r_{n+1}^m \cdot T(s_n^m, a_n^m, s_{n+1}^m) =  RS(\tau_m)\\
\end{align*}
\end{proof}
Assuming $T_o$ is representative sample of the reward dynamics of $R$, in an NMRDP, by Theorem 1, the ARMDP is a suitable proxy for maximizing reward expectations for the NMRDP. 

\section{Experiment 1 Specifications}\label{app:exp1}
For all tasks in Experiment 1, we used $L:S \rightarrow 2^P$ and $j \geq i$ toggle construction. We adapted the Officeworld environment from https://github.com/RodrigoToroIcarte/reward\_machines.git. We noticed that in this implementation, the $*$ tile was a terminal state but resulted in a neutral reward of 0. We instead marked all $*$ tiles as terminal states with a reward of -1 to deter the agent from stepping onto them. Hyper parameters included $\gamma=0.95$, $\alpha=0.1$ and $\epsilon = 0.1$. The only exception is for Task (b) where we used $\epsilon = 0.2$ until the first conflict, to promote exploring. This is because the shortest path to the goal state included the necessary tile to trigger the non-Markov reward. As such, the Markov representation under $S$ and $A$ appeared sufficient to demonstrate optimal behavior. Even with $\epsilon=0.2$, for Task (b), 100,000 episodes was oftentimes not enough to trigger a conflict. To further accommodate this we put a sliding window on the total episode counts used to produce Figure 2, in the main paper, such that the 100,000 episodes contained the first conflict. This was strictly to produce a consistent figure over 10 randomly seeded trials. We include an extended version of Figure 2 from the main paper in Figure \ref{fig:exp1} containing information on $|U|$, $|T_{o}|$, as well as Cumulative Solve Times (in seconds). Machine specifications used in the experiment: 32 core Intel(R) Xeon(R) CPU E5-2640 v3 @ 2.60 GHz with 512 GB of Memory.

\begin{figure}
    \centering
    \includegraphics[width=0.9\textwidth, height=0.35 \textheight]{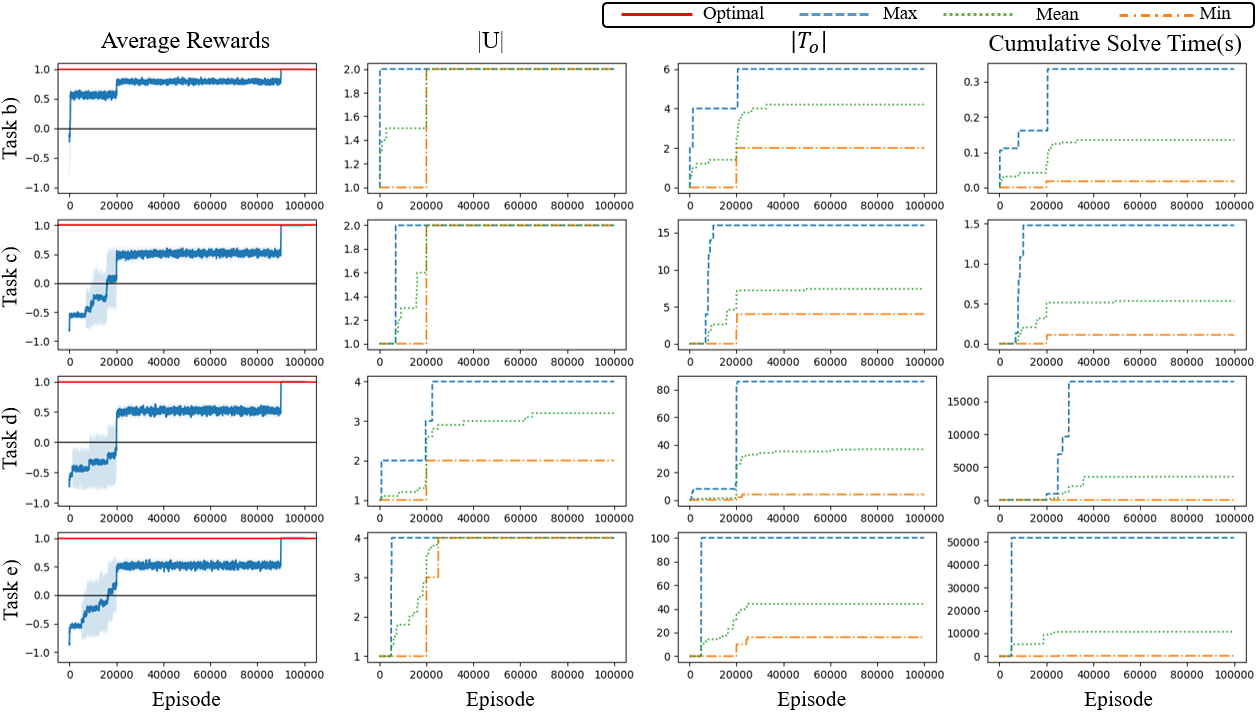}
    \caption{Results of ARMDPQ-Learning evaluated over 10 random trials on Tasks (b-e) in the Officeworld domain. Average Rewards column is displayed with standard deviations. $|U|$ column shows current assumed size of $U$. $|T_{o}|$ column shows current size of $T_{o}$. Cumulative Solve Time(s) represents the total solve time for all ARMDP inferences, in seconds.}
    \label{fig:exp1}
\end{figure}

\section{Experiment 2 Specifications}\label{app:exp2}
Each model in Experiment 2 was trained using a Double Deep Q-Network (DDQN) architecture, meaning that it leveraged a \emph{target} and \emph{non-target} model. All \emph{target} models were trained slowly off of their respective \emph{non-target} model using a smoothing factor of 0.01. Each model across the types (ADQN, DQN, RDQN, GRUDQN, and LSTMDQN) were all trained using a learning rate of 1e-3 with Adam optimizer. ADQN and DQN used a batch size of 32, whereas the RDQN, GRUDQN, and LSTMDQN models used a batch update by stacking a single trajectory. Other hyper parameters included $\gamma = 0.95$ and $\epsilon = 0.1$. Machine specifications used in the experiment: 32 core Intel(R) Xeon(R) CPU E5-2640 v3 @ 2.60 GHz with 512 GB of Memory.

\section{Experiment 3 Specifications} \label{app:exp3}
For all tasks in Experiment 3, we used $L:S \rightarrow 2^P$ and $j \geq i$ toggle construction. Hyper parameters included $\gamma=0.95$, $\alpha=0.1$ and $\epsilon = 0.1$. We applied a similar sliding window, as we did in Experiment 1, but under a lens of 50,000 episodes. We include an extended version of Table 1 in Figure \ref{fig:exp3}. Machine specifications used in the experiment: 32 core Intel(R) Xeon(R) CPU E5-2640 v3 @ 2.60 GHz with 512 GB of Memory.

\begin{figure}
    \centering
    \includegraphics[width=0.9\textwidth, height=0.35 \textheight]{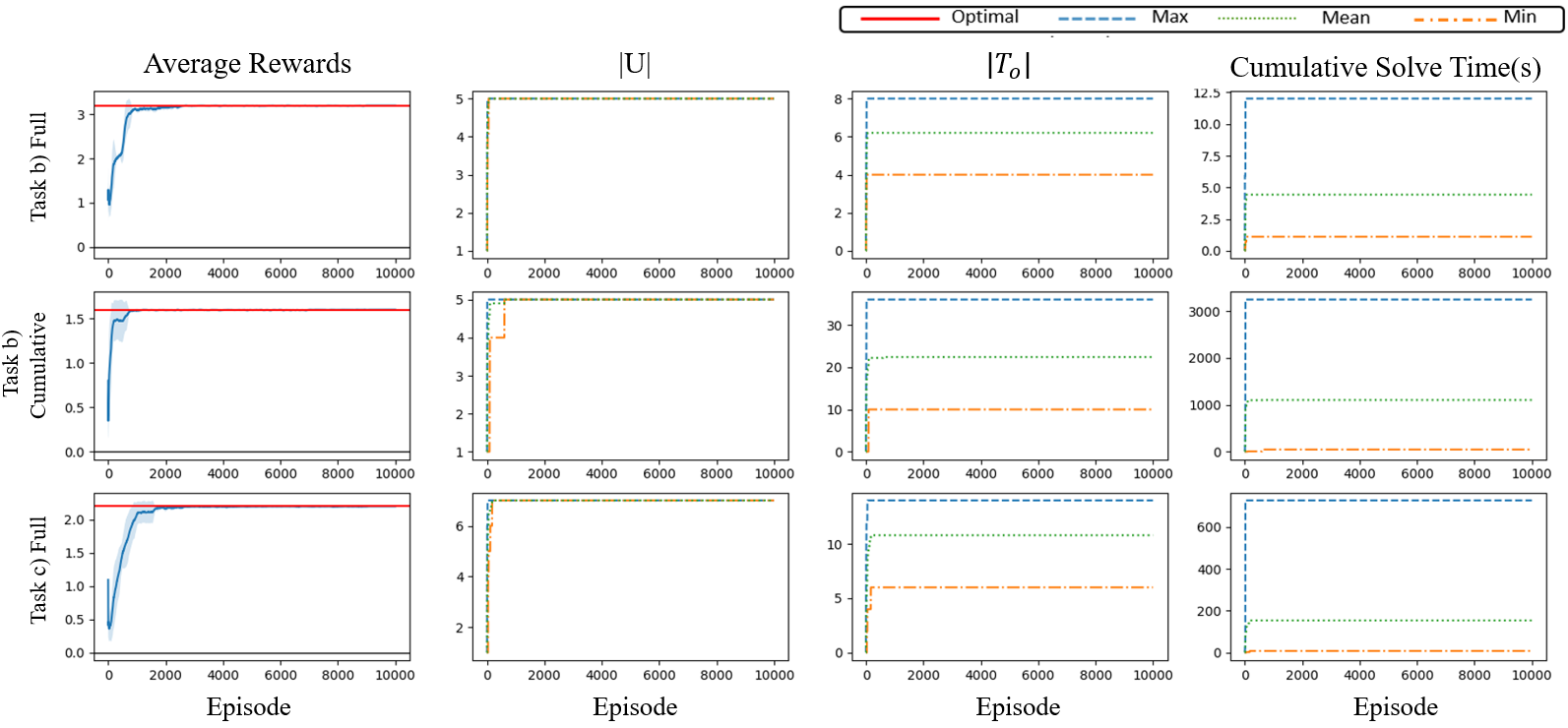}
    \caption{Results of ARMDPQ-Learning evaluated over 10 random trials on Tasks (b) and (c) from the Breakfastworld domain. Task (b) and (c) Full show results when using the full RM for Task (b) and (c). Task (b) Cumulative shows results for using the RM for Task (b) where all rewards are 0 except the terminal cumulative rewards. Average Rewards column is displayed with standard deviations. $|U|$ column shows current assumed size of $U$. $|T_{o}|$ column shows current size of $T_{o}$. Cumulative Solve Time(s) represents the total solve time for all ARMDP inferences, in seconds.}
    \label{fig:exp3}
\end{figure}

\section{Code}
All code available at https://github.com/GregHydeDartmouth/ARMDP\_RL.git

\end{document}